\documentclass{article}

\usepackage{microtype}
\usepackage{graphicx}
\usepackage{subcaption}
\usepackage{booktabs} 

\usepackage{hyperref}

\usepackage[preprint]{icml2026}

\usepackage{amsmath}
\usepackage{amssymb}
\usepackage{mathtools}
\usepackage{amsthm}

\usepackage[capitalize,noabbrev]{cleveref}

\theoremstyle{plain}
\newtheorem{theorem}{Theorem}[section]
\newtheorem{proposition}[theorem]{Proposition}

\theoremstyle{definition}

\theoremstyle{remark}

\usepackage[textsize=tiny]{todonotes}

\icmltitlerunning{Visual Language Hypothesis}

\begin{document}

\twocolumn[
  \icmltitle{Visual Language Hypothesis}
  \icmlsetsymbol{equal}{*}
  \begin{icmlauthorlist}
    \icmlauthor{Xiu Li}{ind}
  \end{icmlauthorlist}

  \icmlaffiliation{ind}{Independent Researcher}

  \icmlcorrespondingauthor{Xiu Li}{lixiulive@gmail.com}
  \icmlkeywords{Machine Learning, Visual Representation Learning}

  \vskip 0.3in
]



\printAffiliationsAndNotice{}  

\begin{abstract}
We study visual representation learning from a structural and topological
perspective. We begin from a single hypothesis: that visual understanding
presupposes a semantic language for vision, in which many perceptual
observations correspond to a small number of discrete semantic states.
Together with widely assumed premises on transferability and abstraction in
representation learning, this hypothesis implies that the visual observation
space must be organized in a fiber–bundle–like structure, where nuisance
variation populates fibers and semantics correspond to a quotient base space.

From this structure we derive two theoretical consequences. First, the
semantic quotient $X/G$ is not a submanifold of $X$ and cannot be obtained
through smooth deformation alone, semantic invariance requires a
\emph{non-homeomorphic, discriminative target}—for example, supervision via
labels, cross-instance identification, or multimodal alignment that supplies
explicit semantic equivalence.

Second, we show that approximating the quotient also places structural demands
on the model architecture. Semantic abstraction requires not only an external
semantic target, but a representation mechanism capable of supporting
\emph{topology change}: an ``expand–and–snap'' process in which the manifold is
first geometrically expanded to separate structure and then collapsed to form
discrete semantic regions.

We emphasize that these results are interpretive rather than prescriptive:
the framework provides a topological lens that aligns with empirical
regularities observed in large-scale discriminative and multimodal models,
and with classical principles in statistical learning theory.

\end{abstract}

\section{Introduction}

Recent progress in visual representation learning has been driven by
large-scale pretraining~\cite{radford2021learning,jia2021scaling}, self-supervised objectives~\cite{he2022masked,simeoni2025dinov3}, and architectural scaling~\cite{vaswani2017attention,dosovitskiy2020image}. Yet a fundamental question remains: what structural properties must a representation possess in order to support semantic abstraction, rather than merely encoding appearance or local regularity?

In this work we approach the problem from a theoretical perspective. We
begin from a single hypothesis: that visual understanding presupposes a
\emph{semantic language for vision}, in which many perceptual observations
correspond to a small number of discrete semantic states. Placed alongside
two widely assumed premises in modern representation learning—transferability
of representations across tasks, and the existence of an abstraction map
from observations to latent representations—this hypothesis implies a
specific geometric organization of visual space.

If semantic identity is invariant under nuisance transformations, then
visual observations must be structured in a fiber–bundle–like form, where
nuisance variation populates high-dimensional fibers and semantics correspond
to a quotient base space. From this structure, we derive two theoretical
requirements for semantic abstraction.

\textbf{First}, the semantic quotient \(X/G\) is not a submanifold of the
observation space \(X\). It arises from collapsing entire nuisance orbits
into single semantic elements, a process that cannot be achieved by smooth
deformation alone. Objectives such as reconstruction or local
self-supervision reshape geometry while preserving the homotopy type of
\(X\), and therefore remain confined to the fiber geometry. This implies
that semantic invariance requires a \emph{non-homeomorphic, discriminative
target}—for example, supervision via labels, cross-instance equivalence, or
multimodal alignment that supplies explicit semantic structure.

\textbf{Second}, approximating the quotient places structural demands not only
on supervision, but also on the model itself. A representation must possess
the architectural capacity to support \emph{topology change}: an
``expand–and–snap'' process in which the manifold is first expanded to
separate structure geometrically, and then selectively collapsed to form
discrete semantic regions. In this view, geometry-preserving architectures
can shape appearance manifolds, whereas architectures capable of routing,
gating, or selective collapse are able to approximate quotient-like semantic
organization.

Our goal is not to evaluate learning paradigms, but to interpret a range of
empirical observations as structural consequences of these requirements. The
analysis suggests that semantic abstraction emerges when two conditions align:
\begin{itemize}
    \item the presence of a discriminative, non-homeomorphic target,
    \item architectural capacity to reshape manifolds both geometrically and
topologically, rather than producing purely geometry-preserving encodings of
appearance.
\end{itemize}

The remainder of the paper formalizes this argument: Section~\ref{sec:vlh}
develops the fiber bundle structure implied by the hypothesis; Secion~\ref{sec:nuisance} explains that the structure complexity of the equivalence group and abstraction for visual semantics; Sections~\ref{sec:generative} analyze topology-preserving
objectives; Section~\ref{sec:discrimitive} examines discriminative and
multimodal supervision as mechanisms enabling quotient formation; and
Section~\ref{sec:conclusion} discusses broader implications for semantic abstraction.

\section{Visual Language Hypothesis}
\label{sec:vlh}

In this work, we do not begin from the limitations of any particular learning objective or model class. Instead, we start from a structural hypothesis about visual understanding itself. We introduce a single new assumption—the Visual Language Hypothesis—and show that, when combined with standard premises already accepted in representation learning, it necessarily induces a fiber-bundle structure on visual observation space.

\subsection{Hypothesis I: The Existence of Semantic Naming}
We hypothesize that for visual intelligence to exist, there must be an \textbf{Equivalence Group} ($G$) acting on visual signals. Without $G$, the visual field is a disjointed sensory stream; with $G$, the image becomes a structured manifestation of a persistent identity.

In this framework, a \textbf{Semantic} ($\ell$) is defined as a \textbf{Named Equivalence Group}. The ``naming'' is the bridge to a symbolic language, transforming an abstract symmetry into a discrete anchor for reasoning.

Let the system be defined by:
\begin{itemize}
    \item $X$: The space of visual observations (The Total Space).
    \item $G$: The group of transformations (The Nuisance) governed by Physics (e.g., $SO(3)$, $S^2$) or Artifacts.
    \item $\ell \in \mathcal{L}$: The named identity (The Base Point).
\end{itemize}

The existence of naming implies that the chaos of $X$ is partitioned into stable, persistent orbits. An observation $x \in X$ is not a standalone object but a realization of $\ell$ under the influence of $g \in G$.

\subsection{Hypothesis II: Transferability via Semantic Compactness}
We posit that the transferability of a representation is determined by its ability to anchor semantics in a \textbf{Finite Basis} of irreducible primitives. We propose a \textbf{Prime Abstraction} model (inspired by Gödel numbering), where fundamental semantics are treated as ``Primes'' $\{p_1, p_2, \dots\}$ of a deep visual language.

\begin{equation}
    x \cong \left( \prod p_i^{a_i} \right) \cdot g
\end{equation}

For a representation to be transferable across tasks and domains, the semantic space $\mathcal{L}$ must be a \textbf{Finite Set} of these primes.
\begin{itemize}
    \item \textbf{Persistence:} Primes provide an immutable identity that does not ``drift'' as it would in a continuous, unstructured manifold.
    \item \textbf{Composition:} Transferability is the act of recognizing these finite primes in novel nuisance environments. If the basis were infinite or continuous, knowledge would be domain-locked.
\end{itemize}

\subsection{Hypothesis III: Existence and Approximability of the Semantic Map}
We posit the existence of a \textbf{Factorization Operator} (or Abstraction Map) $\pi$ that maps the observation space $X$ onto the semantic basis $\mathcal{L}$.
\begin{equation}
    \pi: X \to \mathcal{L}
\end{equation}
The goal of learning is to construct $\pi$ such that it satisfies the equivalence relation:
\begin{equation}
    \forall x \in X, \forall g \in G: \pi(g \cdot x) = \pi(x)
\end{equation}
This operator acts as a ``topological filter'' that extracts the persistent prime from the composite signal. The existence of $\pi$ is the necessary precondition for any visual understanding or inference; without this operator, the system remains a pixel accountant rather than an intelligence.



\subsection{Derivation of the Fiber Bundle Geometry}
\label{sec:fiber}

Given these hypotheses, the geometric structure of the problem space is strictly determined. We do not arbitrarily model vision as a fiber bundle~\cite{steenrod1999topology}; rather, the requirements of intelligibility and transferability \textit{force} the data manifold to adopt this topology.

Formally, the abstraction map $\pi: \mathcal{X} \to \mathcal{L}$ induces a canonical equivalence relation on the observation space. We identify a structural group $G$ (the \textit{Equivalence Group}) acting on $\mathcal{X}$, which represents the set of all semantic-preserving transformations (e.g., camera rotations, illumination changes).
\begin{equation}
    x \sim x' \iff \exists g \in G \text{ such that } x' = g \cdot x
\end{equation}
For every atomic semantic concept $\ell \in \mathcal{L}$, the inverse image defines the fiber $\mathcal{F}_{\ell}$:
\begin{equation}
    \mathcal{F}_{\ell} := \pi^{-1}(\ell) \cong G
\end{equation}
The total observation space $\mathcal{X}$ thus decomposes into a union of these fibers, constituting a principal fiber bundle $(\mathcal{X}, \mathcal{L}, \pi, G)$.\footnote{In practice, this group action may not be globally free or transitive, and fibers may not be identical. Therefore, we regard the semantic structure of perception as fiber-like—that is, a fibration in which fibers may vary and only approximate group-orbit structure. The “principal bundle” interpretation is thus treated as a modeling idealization, not a strict geometric claim.}

The presence of continuous transformation components (such as pose, e.g., subgroups related to $SO(3)$) suggests that the perceptual–semantic structure may not globally trivialize in realistic settings. However, non-triviality does not follow automatically from the topology of $G$. Whether the bundle is trivial or non-trivial depends on the transition structure induced by the data distribution and representation. Accordingly, we treat non-triviality as a hypothesis about real visual semantics, rather than as a derived topological theorem.

\section{Topology Requirement for Semantic Extraction Approximation}
\label{sec:nuisance}
\subsection{Semantic Equivalence Under Nuisances}

Let $X$ denote the observation space (e.g., the manifold of natural images),
and let semantics be defined up to physically meaningful nuisance variations
in the visual generative process.

\textbf{Assumption (Physical Structure Group for Vision).}
For general visual intelligence, semantic identity must be invariant under
at least the following classes of transformations:

\begin{itemize}
    \item Rigid viewpoint changes, modeled by $SO(3)$ acting on the scene geometry.
    \item Illumination direction and intensity, modeled (to first order) by the sphere $S^2$
    of lighting directions together with photometric scaling.
    \item Additional nuisance factors such as mild nonrigid deformation,
    articulation, occlusion, and background variation.
\end{itemize}

We write $G$ for the resulting (potentially semi-direct) group of visual
equivalences acting on $X$. Two observations $x_1,x_2\in X$ are semantically
equivalent if and only if they belong to the same $G$-orbit,
\[
x_2 \in \mathcal{O}_{x_1} := \{ g \cdot x_1 \mid g \in G \},
\]
so that the semantic space is modeled by the quotient $X/G$ with projection
$q:X\to X/G$.


\subsection{Abstraction as Quotient Factorization}

A semantic abstraction operator is a $G$-invariant map
\begin{equation}
\pi : X \to \mathcal{L},
\qquad
\pi(x) = \pi(g\cdot x)\ \ \forall g \in G,
\end{equation}
so that $\pi$ factors through the quotient:
\begin{equation}
\pi = h \circ q,
\qquad h : X/G \to \mathcal{L}.
\end{equation}
We say that $\pi$ is \emph{topologically faithful} if $h$ is a homeomorphism
onto its image, i.e., $\mathcal{L} \cong X/G$ as topological spaces. In that case
the learned representation preserves the global structure of the semantic
quotient rather than merely collapsing it.

\subsection{Hardness of $G$ in General Visual Intelligence}

\textbf{Claim (Worst-Case Canonicalization Lower Bound).}
For the physically grounded structure group $G$ described above,
any exact and topologically faithful abstraction operator $\pi$
induces a solution to an \emph{orbit canonicalization problem}
for the joint action of
\[
SO(3) \times S^2 \ (\text{plus additional nuisance subgroups})
\]
on $X$.

Therefore, in the worst case, evaluating $\pi(x)$ is at least as hard
as deciding orbit equivalence and (implicitly) recovering a canonical
representative of the orbit $\mathcal{O}_x$ under $G$:
\[
\mathrm{Complexity}(\pi)
\;\succeq\;
\mathrm{Complexity}(\text{OrbitCanonicalization}_{G}).
\]

\textbf{Explanation.}
Robust semantics require invariance to viewpoint and illumination.
But $SO(3)$ and $S^2$ generate highly nontrivial orbit geometries:
distinct combinations of pose, lighting, and occlusion can produce
images that are metrically far apart in $X$ while belonging to the
same semantic orbit.

If $\pi$ is faithful to $X/G$, then
\[
x \sim_G y
\;\Longleftrightarrow\;
\pi(x)=\pi(y),
\]
so computing and comparing $\pi(x)$ solves orbit membership.
In this sense, semantic abstraction must implicitly resolve (at least
up to equivalence) the nuisance transformations generated by $G$.

For realistic scenes, these orbits are high-dimensional, nonconvex,
and often contain self-intersections arising from occlusion and shadowing.
Hence, even before considering learning or finite capacity, the quotient 
induced by $G$ has nontrivial geometric and topological structure.
\subsection{From Hard Semantic Equivalence to Latent Manifold Shaping}

Rather than computing $\pi$ directly on the observation manifold $X$,
a deep network implements an \emph{approximate abstraction}
$\hat{\pi}$ by composing a sequence of continuous operators that
progressively deform $X$ into a learned latent manifold $Z$,
\[
X \xrightarrow{\ \phi\ \ } \mathcal Z \xrightarrow{\ \ell\ } \mathcal L,
\qquad
\hat{\pi} = \ell \circ \phi .
\]
The encoder $\phi$ reshapes the geometry of $X$ so that semantic
equivalence becomes easier to express in the latent domain $\mathcal Z$.

\textbf{Assumption (Semantic Simplicity in Latent Space).}
For a representation to be considered good, we assume that, in $\mathcal Z$,
semantic abstraction becomes \emph{simple}: different semantic classes
can be separated by a polyhedral tessellation generated by the linear
readout. Equivalently, the latent manifold $\mathcal Z$ is shaped so that
semantic regions form a Voronoi–type convex structure,
\begin{equation}
\mathcal Z = \bigcup_k \mathcal{R}_k,
\qquad
\mathcal{R}_k \text{ convex}.
\end{equation}

\textbf{Interpretation.}
The role of the network is not to solve the hard quotient problem on $X$
directly, but to deform the observation manifold into a latent manifold
$Z$ whose geometry supports a simple semantic lattice. The complexity of
$G$ is thus absorbed into the deformation $\phi$, while the final mapping
$\ell$ acts as a combinatorial tessellation over a convex partition of $\mathcal{Z}$.

\paragraph{Remark: Relation to Classical Learning Theory.}
The Voronoi–type convex regions that appear in the latent manifold $Z$
are not an artifact of our interpretation, but follow directly from the
geometry of multinomial logistic regression and Softmax classification.
The linear readout layer induces affine logit comparisons, and the MAP
decision rule reduces to a partition of $Z$ into convex polyhedral cells.
Thus, Voronoi convexity reflects the implicit decision geometry assumed
by standard probabilistic classifiers.

\section{Generative Modeling as Manifold Shaping}
\label{sec:generative}
\textbf{Idea.}
Generative and self-supervised objectives encourage the network to
reshape the geometry of the observation manifold while preserving its
homotopy type; the manifold is bent, smoothed, or reparameterized, but
its global topology remains unchanged.

\subsection{Reconstruction Loss as Homotopy Preservation}

We first formalize the idea that reconstruction-based generative models
(autoencoders, VAEs, diffusion decoders with L$_2$ reconstruction terms)
learn a near-identity deformation of the data manifold, and therefore
preserve its topological type.

Let $X$ be a compact subset of $\mathbb{R}^n$ (the data manifold). An
autoencoder consists of an encoder $f:X\to \mathcal{Z}$ and a decoder $g:Z\to \mathbb{R}^n$
trained to minimize a reconstruction loss of the form
\[
\mathcal{L}_{\mathrm{rec}}(f,g)
=
\mathbb{E}_{x\sim \mathcal{D}}
\big[
\ell\big(g(f(x)), x\big)
\big],
\]
where $\ell(\hat{x},x)$ is continuous in both arguments and penalizes
deviations between $\hat{x}$ and $x$ (e.g.\ $\|\hat{x}-x\|^2$).

\begin{proposition}[Reconstruction Loss as Homotopy Preservation]
\label{prop:reconstruction-homotopy}
Assume:
\begin{enumerate}
    \item $f$ and $g$ extend to continuous maps on an open neighborhood
    of $X$ in $\mathbb{R}^n$;
    \item the reconstruction error is uniformly small,
    \(
        \sup_{x\in X} \|g(f(x)) - x\| \le \varepsilon,
    \)
    with $\varepsilon>0$ small enough so that the straight-line segment
    between $x$ and $g(f(x))$ remains inside a tubular neighborhood of $X$;
    \item this tubular neighborhood retracts continuously onto $X$.
\end{enumerate}
Then the map $T := g\circ f : X\to X$ is homotopic to the identity
$\mathrm{Id}_X$. In particular, $T$ preserves the homotopy type of $X$:
it cannot create or destroy connected components, loops, or higher-order
topological features.
\end{proposition}

\begin{proof}[Sketch of proof]
Consider the straight-line homotopy
\[
H_t(x) = (1-t)\,x + t\,T(x),
\qquad t\in[0,1].
\]
By assumption (2), $T(x)$ lies within distance $\varepsilon$ of $x$,
so $H_t(x)$ remains in the tubular neighborhood of $X$ for all $t$.
By assumption (3), there exists a continuous retraction
$r$ from this neighborhood onto $X$. Composing, we obtain a homotopy
\[
\widetilde{H}_t(x) := r\big(H_t(x)\big),
\]
which is continuous in $(x,t)$, satisfies $\widetilde{H}_0(x)=x$ and
$\widetilde{H}_1(x)=r(T(x))$. Since $T(x)\in X$ and $r$ restricts to
the identity on $X$, we have $r(T(x)) = T(x)$, hence $\widetilde{H}_1=T$.
Thus $\widetilde{H}_t$ is a homotopy from $\mathrm{Id}_X$ to $T$.
\end{proof}

\textbf{Corollary.}
If, in addition, the restriction of $g$ to $f(X)$ is continuous and
$g\circ f$ is uniformly close to $\mathrm{Id}_X$, then $f:X\to f(X)$
is a homotopy equivalence with $g|_{f(X)}$ as a homotopy inverse.
Therefore $f(X)$ has the same homotopy type as $X$.

\textbf{Implication.}
Reconstruction-driven generative training encourages $f$ and $g$ to
form a near-identity deformation of $X$. Topologically, this means
that the encoder-decoder pair learns a representation that is
homotopy-equivalent to the original observation manifold:
the global topology (connected components, loops, holes) of $X$
is preserved. As a consequence, the semantic quotient $X/G$,
which requires collapsing entire $G$-orbits into single points,
cannot be realized by such a homotopy preserving near identity map.
Generative models may \emph{bend} or \emph{smooth} $\mathcal{X}$, but they
do not perform the non-homeomorphic quotient necessary for
semantic abstraction.

\subsection{Contrastive Loss as Local Metric Shaping}

We next consider contrastive self-supervision~\cite{chen2020simple,he2020momentum},
where an encoder $f_\theta:X\to Z$ is trained to bring positive pairs
(close in some augmentation sense) together in $Z$ while pushing negative
pairs apart, without an explicit reconstruction term.

A generic contrastive loss takes the form
\[
\mathcal{L}_{\mathrm{ctr}}(\theta)
=
\mathbb{E}
\Big[
\ell\Big(
\big\langle f_\theta(x), f_\theta(x^+)\big\rangle,
\big\{\langle f_\theta(x), f_\theta(x^-_j)\rangle\big\}_j
\Big)
\Big],
\]
where $(x,x^+)$ is a positive pair (e.g.\ an augmentation of the same
sample) and $\{x^-_j\}_j$ are negatives. The loss $\ell$ is continuous
in its arguments, and $f_\theta$ is a neural network, hence continuous
in both $x$ and $\theta$.

\begin{proposition}[Contrastive Training as Continuous Deformation]
Let $\theta_t$ be the parameter trajectory obtained by gradient-based
optimization of $\mathcal{L}_{\mathrm{ctr}}$, and define
$f_t := f_{\theta_t}:X\to Z$. Assume that for all $t\in[0,1]$ the map
$f_t$ is continuous and remains an embedding of $X$ into $Z$ (no
topological self-crossings). Then the family $\{f_t\}_{t\in[0,1]}$
defines a homotopy of embeddings between $f_0$ and $f_1$, and the
homotopy type of $f_t(X)$ is constant in $t$.
\end{proposition}

\begin{proof}[Sketch]
Continuity of $\theta_t$ and of the network with respect to its
parameters implies that $(x,t)\mapsto f_t(x)$ is continuous. For each
fixed $t$, $f_t$ is an embedding by assumption, so $f_t(X)$ is
homeomorphic to $X$. The map
\[
H(x,t) := f_t(x)
\]
defines a homotopy between $f_0$ and $f_1$. Since homotopy equivalence
is preserved under embeddings and $X$ is fixed, all $f_t(X)$ share the
same homotopy type as $X$.
\end{proof}

\textbf{Interpretation.}
Contrastive learning continuously deforms the embedding of $X$ inside
$Z$ by \emph{reshaping distances}:

\begin{itemize}
    \item positive pairs $(x,x^+)$ are pulled closer in $Z$,
    \item negative pairs $(x,x^-)$ are pushed apart.
\end{itemize}

Geometrically, this is a modification of the \emph{local metric}
on $X$ induced by $f_t$ (e.g.\ the pullback of the Euclidean metric
on $Z$), rather than a change in the global topology of $X$. As long
as $f_t$ remains an embedding, contrastive training does not merge
distinct points or collapse entire $G$-orbits into single semantic
classes; it only adjusts their relative positions in $Z$.

In particular:

\begin{itemize}
    \item Augmentations $x^+ = a\cdot x$ for $a$ in an augmentation group
    encourage \emph{local invariance} (small distance between $f(x)$ and
    $f(x^+)$), but the orbit $\mathcal{O}_x$ is still represented as a
    \emph{thick} set in $\mathcal{Z}$, not a single point. 
    \item Negative sampling prevents trivial collapse of all orbits to
    a single embedding, but does not enforce a global quotient $X/G$.
\end{itemize}

Thus, contrastive objectives primarily reshape the \emph{Riemannian
geometry} of the data manifold (angles, lengths, local neighborhoods)
without altering its homotopy type. Semantics may become easier to read
out \emph{locally} (for example, via linear probes on well-separated
regions of the embedding), but the global semantic quotient structure
$X/G$ is not explicitly formed. Because the action of $G$ induces
non-trivial fibers, the manifold cannot be globally trivialized: semantic
and nuisance factors remain entangled in the latent space $\mathcal{Z}$ at
the topological level, even when they appear partially separable in local
neighborhoods.

This helps explain the observed difficulty in transferability. A
representation that lacks a globally consistent quotient organization may
support task-specific separability in some regions, but fails to maintain
semantic invariance across tasks, domains, or contexts, since the same
semantic class may occupy different locations along the fiber structure.




\section{Discriminative Modeling as Topological Collapse}
\label{sec:discrimitive}
\textbf{Idea.}
Discriminative objectives introduce explicit identification constraints
across samples, driving the representation toward a quotient-like
collapse of $G$-orbits into shared decision regions.

The previous sections showed that reconstruction-driven and contrastive
objectives preserve the homotopy type of the data manifold: they deform
$X$ continuously, but do not perform the quotient operation required to
collapse $G$--orbits into semantic classes.

In contrast, \emph{discriminative} objectives (classification, GAN
discrimination, contrastive alignment across modalities, etc.) introduce
explicit constraints that identify distinct observations as belonging to
the same semantic decision region.

Let $\hat{\pi} = \ell \circ \phi$ denote a representation with a
decision operator $\ell$ (e.g., linear logits + Softmax). A
discriminative loss enforces constraints of the form
\[
\ell(\phi(x_i)) = \ell(\phi(x_j))
\quad\text{whenever}\quad
x_i \sim_G x_j,
\]
or, more weakly, that $x_i$ and $x_j$ map into the same decision region.
Such constraints require the representation to
\emph{identify entire subsets of $X$}. Topologically, this corresponds
to a quotient map that collapses fibers, mapping them to the same decision cell.

\textbf{Consequence.}
Unlike reconstruction or contrastive training, discriminative objectives
do not preserve the homotopy type of $X$: they drive the network toward
representations in which semantically equivalent orbits are mapped into
a single convex decision region. The topology is altered not by a smooth
deformation, but by a \emph{piecewise identification of regions} induced
by the classifier.

\subsection{Attention and Softmax as Piecewise Routing}

Standard feedforward networks with ReLU or smooth activations implement
continuous piecewise-linear maps. As such, they deform the data manifold
continuously: they can bend and fold, but cannot introduce explicit
topological identification except at degenerate singular layers.

Self-attention with Softmax behaves differently. Given a query $q$ and a
set of keys $\{k_i\}$,
\[
\alpha_i = \mathrm{softmax}(\langle q, k_i\rangle),
\qquad
y = \sum_i \alpha_i v_i,
\]
the Softmax operator concentrates mass onto a small subset of tokens as
the logits separate. In the low-temperature or high-margin regime,
attention becomes \emph{selective routing}: neighborhoods of inputs that
produce different dominant attention patterns are mapped to
qualitatively different computational branches.

Geometrically, this induces a \emph{piecewise} structure in $Z$:
\[
X = \bigcup_r U_r,
\qquad
\phi(x) = \phi_r(x) \ \text{for } x\in U_r,
\]
where each region $U_r$ corresponds to a distinct attention-routing
pattern. Transitions between regions occur along nonsmooth boundaries and
act as \emph{effective tears} in the latent manifold.

\textbf{Interpretation.}
Whereas ReLU networks perform smooth (homotopy-preserving) deformations,
attention with Softmax introduces region-wise functional switching. This
enables the network to
\begin{itemize}
    \item merge distant parts of $X$ into the same semantic region,
    \item collapse subsets of a $G$--orbit into a single representation,
    \item implement nontrivial quotient-like identifications.
\end{itemize}

This mechanism explains why architectures with attention (e.g.\ ViT)
exhibit a greater capacity to approximate semantic quotients than purely
feedforward continuous encoders: the Softmax attention layer supplies the
discrete routing behavior required to approximate topological collapse,
rather than merely deforming the manifold continuously.

\subsection{Non-Homeomorphic Target Requirement}

Generative and self-supervised objectives reshape the
geometry of the observation manifold $X$ through continuous
deformations. Reconstruction-based training encourages near-identity
mappings, while contrastive training modifies the local metric or
induces approximately isotropic (often hyperspherical) embeddings.
In both cases, the resulting representation remains homotopy-equivalent
to $X$: the manifold is bent, smoothed, or reparameterized~\cite{kingma2013auto}, but its
global topology is preserved.

Semantic abstraction, however, is modeled not by $X$ itself but by the
quotient space $X/G$, where $G$ represents nuisance transformations such
as viewpoint, illumination, or articulation. The quotient $X/G$ is not a
submanifold of $X$: it is obtained by collapsing entire $G$-orbits into
single equivalence classes, a process that identifies points that may be
far apart in $X$ and introduces singularities and strata in the resulting
space. Topologically, this collapse is non-invertible and cannot be
realized by a smooth deformation of $X$.

\textbf{Implication.}
Any learning objective that preserves the homotopy type of $X$ cannot
recover $X/G$. It may produce useful geometric structure, enable sampling,
regularize density, or improve separability, but it does not perform the
topological identification required for semantic invariance. The semantic
target is therefore \emph{non-homeomorphic} to the observation manifold:
it is not ``contained within'' $X$, but instead corresponds to a
different space obtained through quotient collapse.

To approximate such a space, the learning process must incorporate a
signal that breaks the topology of $X$—for example, discriminative
constraints or alignment with another modality. Labels, cross-instance
supervision, or image–text correspondence supply relations that cut
across $G$-orbits, forcing the representation to merge points that are
distinct in $X$ but equivalent in $X/G$. These signals provide the
topology-changing mechanism that continuous reconstruction or contrastive
objectives, by design, do not introduce.

\textbf{Interpretation.}
From this viewpoint, recent progress in multimodal and discriminative
models is not accidental: such systems succeed at semantic abstraction
because their training objectives introduce a non-homeomorphic target
space. By coupling $X$ to an external structure that is not topologically
equivalent to $X$ itself, the model is able to collapse $G$-orbits and
form a representation aligned with the semantic quotient, rather than a
topology-preserving reparameterization of the observation manifold.

\section{A Toy Example Through the Fiber--Bundle Perspective}

In the previous section, we introduced the fiber--bundle formulation of perception,
where a signal space $X$ is organized into fibers by a projection
$\pi : X \rightarrow \mathcal L$ onto a semantic base space $\mathcal L$, and where
$G$ denotes the induced equivalence relation capturing intra--fiber
variability. In this section, we construct a controlled example within
the same formal structure. Our purpose is not to disprove existing
unsupervised learning paradigms, but rather to understand \emph{what}
these methods learn when semantics is explicitly represented as a
quotient structure.

\subsection{A Constructed Bundle with Explicit Semantics}

Let $A,B \in \{0,\dots,n-1\}$ and define a semantic quantity
\[
C = (A+B) \bmod n ,
\]
where $(A+B)$ induces a simple interaction that is deliberately
non--symmetric, while the modulo operation confines the outcome to a
finite semantic class. For each pair $(A,B)$ we rasterize a visual
expression such as ``$A+B$'' into an image $x_{A,B} \in X$. The rendering
process may introduce variability due to layout, font, distortion, or
noise; these variations are absorbed into the fiber structure $G$,
while the semantics are declared explicitly by the quotient map
\[
\pi(x_{A,B}) = C .
\]

In this construction, the semantic relation is not inferred implicitly
from data but is built into the generative process itself. One may
interpret this as ``hiding'' the map $\pi$ inside the pixels while
retaining a mathematically precise quotient structure in the background.
Although synthetic, this bundle is not merely a toy abstraction: it
reflects situations in natural imaging where a small set of latent
generative factors produces large families of signal realizations. The
advantage of our setting is that the projection $\pi$ is explicit and
thus analytically accessible.

\subsection{Studying Existing Unsupervised Objectives in the Bundle Geometry}

With this construction in place, we use it as a lens to study the
behavior of several widely used unsupervised objectives, in particular
masked reconstruction and contrastive learning. Our goal is modest: we
do not claim invalidity of these methods, but rather investigate their
characteristics when viewed relative to the semantic projection $\pi$
and the fiber structure $G$.

\paragraph{Reconstruction-Based Objectives (e.g., MAE).}
Masked reconstruction methods minimize a fidelity loss in signal space,
\[
\min_{e,d}\; \mathbb{E}\,\|x - d(e(x))\|,
\]
and may achieve near-perfect reconstruction while remaining entirely
within each fiber $G(x)$. In our bundle formulation, such objectives
can be satisfied by learning statistical regularities of the rendering
process, without requiring any interaction with the quotient map $\pi$.
Thus, there exist optimal solutions whose latent representations provide
no simpler access to the semantic class $C$ than the raw image itself.
From this perspective, the representation primarily captures
\emph{intra--fiber structure} rather than cross--fiber abstraction.

\paragraph{Contrastive Learning Without External Anchoring.}
Instance-level contrastive learning enforces similarity between two
augmentations of the same image and separation between different
instances. Geometrically, this encourages invariance under
transformations that remain within a fiber, but it does not impose any
constraint relating two signals $x_{A,B}$ and $x_{A',B'}$ that share the
same quotient label $(A+B)\bmod n = (A'+B')\bmod n$. The resulting
representation distinguishes instances while remaining insensitive to
the equivalence structure induced by $\pi$. In this sense, the method
models the geometry of fibers but does not operate across them.

\paragraph{Cross-Modal or Truth-Anchored Objectives (e.g., CLIP).}
When an external modality introduces relational or partially semantic
constraints (for example, text correlated with the value of $C$), the
learning objective begins to couple the signal space with the quotient
space. Such objectives do not guarantee complete semantic recovery, but
they differ qualitatively in that they encourage organization relative
to $\pi$, rather than remaining closed within $G$.

\subsection{A Working Hypothesis}

The fiber--bundle formulation suggests a working hypothesis: methods
whose objectives operate primarily within fibers may scale effectively
in modeling statistical structure, yet struggle to scale toward semantic
abstraction, because the objective does not require interaction with the
quotient space $L$. Our constructed example provides a controlled
setting in which this distinction becomes explicit; real visual data may
involve far richer semantic processes, but the structural separation
between fiber variability and semantic projection remains.

\subsection{Why This Construction is Minimal Rather than Arbitrary}

The choice of the synthetic form $(A+B)\bmod \mathbb{Z}_n$ is not
intended as a playful toy, but rather as a \emph{minimal} construction
that already exhibits several structural properties that are central to
semantic abstraction in the fiber--bundle setting.

First, the interaction term $(A+B)$ introduces an intrinsically
stochastic generative structure: the rendered signal is jointly
determined by two latent variables $A$ and $B$, and in general neither
$A$ nor $B$ is recoverable from the semantic quantity $C = (A+B)\bmod n$.
In particular, many distinct pairs $(A,B)$ collapse to the same
quotient value under~$\pi$. This reflects the situation in real visual
data, where multiple generative configurations may lead to an identical
semantic class, and where inversion to latent causes is not uniquely
defined.

Second, $(A+B)$ also makes the construction deliberately
non--symmetric. The two latent factors do not occupy interchangeable
roles, and neither variable can be predicted from the other. This allows
the fiber to contain structured redundancy and variation, rather than
degenerating into a trivially factorized representation. By increasing
the range of $A$ or $B$, one may enlarge the space of generative
configurations within a fiber while keeping the semantic base fixed.

Third, the modulo operation $\mathbb{Z}_n$ confines the outcome to a
finite semantic class space, making the base $L$ amenable to evaluation
and comparison. The projection $\pi$ therefore represents an explicitly
defined classification task, while the fiber $G(x)$ retains the
richness of signal-level variability arising from the generative
degrees of freedom.

In this sense, the bundle $(X,
\mathcal L,\pi,G)$ constructed here should be
understood as a minimal abstraction that preserves three key phenomena:
(i) many-to-one semantic collapse under $\pi$, (ii) irrecoverability of
latent generative factors from the quotient, and (iii) nontrivial
structure and redundancy within fibers. These properties are sufficient
to study how different learning objectives behave with respect to the
semantic projection, without requiring a fully realistic image-formation
pipeline.

\section{A Topological View of Semantic Abstraction}
\label{sec:conclusion}

Our objective is not to compare architectures or judge learning paradigms,
but to understand what it means for a model to perform \emph{semantic
abstraction}—a prerequisite for developing intelligence. In semantic
abstraction, a high-dimensional, continuous perceptual space $X$ must be
mapped to a representation where many distinct observations are assigned
to a \emph{small set of symbolic or conceptual states}. Formally, this is
a collapse from $X$ to a low-cardinality quotient $X/G$, where entire
equivalence orbits of nuisance variation are identified with a single
semantic value.

Such a transformation cannot be realized as a smooth deformation of $X$.
A continuous map preserves the topology of the input manifold, whereas
semantic abstraction requires collapsing large continuous regions into
discrete symbolic states. From a topological perspective, the network must
exhibit two distinct phases of behavior:

\begin{enumerate}
    \item \textbf{Expansion (Untangling).}
    The representation expands and stretches the manifold, increasing
    separation between regions that will ultimately correspond to distinct
    semantic concepts. This phase is largely continuous and resembles
    geometric unfolding or high-dimensional embedding.

    \item \textbf{Snapping (Collapse).}
    At a later stage, the model introduces sharp identifications or routing
    discontinuities that collapse each $G$-orbit into a single semantic
    region. This constitutes a genuine topology change: neighborhoods that
    were distinct in $X$ become identified in the semantic space.
\end{enumerate}

We refer to this process as \emph{expand-and-snap}. Expansion enables the
degrees of freedom necessary to separate semantics, while snapping performs
the topological quotient that converts continuous variation into discrete
symbolic structure.

\paragraph{Cardinality and Topology.}
The challenge of semantic abstraction is not primarily one of
dimensionality, but of cardinality. The observation manifold admits a
high-cardinality continuum of states, whereas the semantic quotient
corresponds to a small discrete set of equivalence classes. Mapping from
the former to the latter requires a topology-changing identification: a
continuous embedding may expand or reshape geometry, but only a
quotient-like collapse can reduce cardinality and produce semantic
symbols.

In this sense, dimensionality belongs to the domain of geometry—concerned
with how a manifold may be embedded or expanded—whereas cardinality is tied
to topology, governing when many continuous states are identified and
collapsed into discrete semantic categories.

\subsection{Design Consequences: Why Modern LLMs Behave the Way They Do}

This expand–snap perspective clarifies the emergent design regularities of
large language and foundation models.

\paragraph{Discriminative Targets}
Despite LLMs are often described as generative models~\cite{brown2020language}, their learning signal is fundamentally \emph{discriminative}: prediction requires selecting
one token from a small, discrete vocabulary. The objective thus enforces a
low-cardinality target space, providing the topological pressure necessary
for symbolic collapse.

\paragraph{Reverse Bottlenecks and Attention}
Transformer architectures~\cite{vaswani2017attention} increase the
internal dimensionality of the representation, enabling the network to
stretch and untangle the manifold prior to semantic collapse. The Softmax
attention and routing mechanisms then realize the snapping phase: they
perform selective pathway activation and mass concentration, approximating
a quotient-like identification of latent regions.

Mixture-of-Experts (MoE) architectures~\cite{dai2024deepseekmoe} amplify
this behavior by explicitly routing inputs into discrete expert
partitions, thereby strengthening the topology-changing effect of the
collapse stage. Recent gated architectures~\cite{qiu2025gated} can be
interpreted in the same light: gating functions operate as controlled
routing operators that approximate the discrete identifications required
for semantic abstraction.

\paragraph{Attention Spikes as Topological Surgery}
The sharp attention peaks and routing saturation observed in large models
should not be interpreted merely as optimization artifacts. Rather, they
are signatures of the network attempting to perform \emph{topological
surgery}: fracturing a continuous latent geometry into the disjoint symbolic
cells needed for semantic abstraction. Gradient-based optimization resists
these discontinuities, producing the well-documented tension between smooth
learning dynamics and discrete representational organization.

\subsection{Topological Lens to Classical Theory}

Our analysis does not seek to reinterpret classical statistical
learning theory, but rather to place it in geometric correspondence. In
particular, we find that the topological perspective developed in this
work resonates naturally with two foundational insights in machine
learning, and we view this alignment as a source of conceptual validation.

\paragraph{Cover’s Theorem (Expansion)} Cover~\cite{cover2006geometrical} showed that the probability of linear separability increases with
the dimensionality of the embedding space. In our framework, this result
appears as the \emph{expansion phase}: moving to a higher-dimensional
representation provides the geometric freedom necessary to untangle and
separate structure prior to semantic collapse. Our interpretation does not
alter Cover’s conclusion, but offers a complementary geometric intuition
for why dimensional expansion facilitates separability.

\paragraph{Vapnik’s Principle (Snapping)} Structural Risk Minimization~\cite{cortes1995support} emphasizes that stable generalization arises from limiting effective model capacity and enforcing margins. In our topological account, this corresponds to the \emph{snap phase}: collapsing
fibers into discrete regions yields large, concept-level margins and
suppresses nuisance variability. Again, our formulation does not modify
the theory, but frames it as the topological dual of quotient formation in
semantic abstraction.

\medskip
We regard these correspondences as reassuring rather than speculative: the
fact that the expand–and–snap picture echoes established learning theory
suggests that our topological formulation is consistent with—and quietly
supported by—the foundational principles on which modern machine learning
was built.

\paragraph{Remark: Why Transformer Architectures are Structurally Advantageous.}
Our analysis suggests that the empirical dominance of Transformer~\cite{vaswani2017attention} models
cannot be attributed solely to engineering and scaling factors. From the
topological perspective developed in this work, the building blocks of the
Transformer architecture—high-dimensional expansion, Softmax attention, and
selective routing—extend the expressive capacity of deep networks beyond
smooth geometric deformation. They enable representations that can reshape
the manifold not only geometrically but also topologically, supporting the
expand–and–snap behavior required for semantic abstraction.

In this sense, Transformers are not merely more scalable; they are
structurally more versatile. Their architectural primitives allow the model
to approximate both continuous manifold shaping and discrete
quotient-forming operations, making them functionally more universal with
respect to the class of transformations involved in semantic representation.

\section{Related Work}

\paragraph{Invariance and Geometric Structure.}
A large body of work studies invariance, disentanglement, and geometric structure in representations~\cite{higgins2017beta, locatello2019challenging, bronstein2021geometric}. Our analysis departs from the prevailing focus on coordinate-based factorization by modeling semantics as a topological quotient space ($X/G$) rather than a set of independent latent coordinates. This perspective helps reconcile the empirical ubiquity of entangled representations with the theoretical difficulty of recovering semantics from local objectives alone. 

Furthermore, while the principle of equivalence is formally grounded in \textit{Gauge Theory} (where nuisance variability acts as a local gauge transformation), our objective is not to explicitly engineer a gauge-equivariant architecture.  Instead, we focus on the \textit{topological obstructions} this symmetry creates. We analyze whether standard networks possess the necessary topological capacity to resolve the non-trivial structure of the gauge quotient, arguing that the enforcement of such invariance creates a fundamental mismatch in generic architectures.

It is also worth noting that our use of quotient structure $X/G$ aligns with the definition of optimal visual representations proposed by Soatto and Chiuso \cite{soatto2014visual}, who derive minimal sufficient statistics using information-theoretic arguments.  However, while the probabilistic framework requires estimating high-dimensional densities—making the analysis of ``hardness'' dependent on specific priors and variational bounds—the topological perspective offers a more direct characterization. By focusing on the global connectivity and homotopy type of $X/G$ rather than local probability masses, we can identify structural obstructions (such as the mismatch between non-contractible orbits and contractible decision regions) that are invariant to density fluctuations.  Thus, the topological view provides a cleaner, geometry-first proof of the representation gap without relying on the intractability of Bayesian inference.

The topological properties of deep representations have been investigated empirically by Naitzat et al.~\cite{naitzat2020topology}. They observe that neural networks operate by progressively simplifying the topology of the data manifold, systematically reducing its Betti numbers until the data becomes topologically trivial (contractible) in the final layers to satisfy linear separability. 
Our analysis extends this observation to the context of semantic abstraction. We argue that this inherent tendency to simplify creates a \textit{topological mismatch} when learning from physically grounded data. 
Since the true semantic quotient $X/G$ retains complex, non-trivial topology (e.g., loops and non-contractible cycles), a network that strictly enforces topological simplification struggles to represent the quotient faithfully.

\paragraph{Reconstruction-Based Learning.}
Autoencoders, denoising models, VAEs, and masked-reconstruction approaches learn
representations by predicting corrupted or missing observations
\cite{kingma2013auto, he2022masked}. These models are highly
effective at capturing appearance statistics and local structure. In our
framework, this behavior follows from the fact that reconstruction objectives
encourage near-identity deformations of the observation manifold, preserving
its homotopy type and emphasizing variation within nuisance fibers rather than
quotient-level abstraction.

\paragraph{Contrastive and Self-Supervised Methods.}
Contrastive learning and related objectives~\cite{oord2018cpc, chen2020simclr,
grill2020bootstrap, bardes2022vicreg} enforce similarity between augmented
positives and separation from negatives. Prior analyses highlight the role of
augmentation and invariance assumptions~\cite{tian2020makes, wang2022ssl}. Our
view is complementary: when positives are defined within augmentation-induced
orbits of $G$, training reshapes geometry but leaves the quotient structure
underdetermined unless external equivalence information is introduced.

\paragraph{Language and Multimodal Supervision.}
Recent multimodal systems demonstrate that coupling vision with language or
other external modalities improves semantic alignment and transfer
\cite{radford2021learning,tschannen2025siglip}. In our interpretation,
such signals act as an external semantic oracle that provides the
non-homeomorphic target required for quotient formation, offering a structural
explanation for their observed effectiveness.

\section*{Acknowledgments}

We are grateful to Prof. Hongdong Li for his pivotal comments on invariance. His perspective on equivalence classes helped resolve the characterization of the nuisance group $G$, effectively supplying the missing component of the proposed topological framework. We also extend our thanks to Dr. Jiashi Feng for his rigorous review and for verifying the logical soundness of the theoretical arguments.

\bibliography{main}

@article{kingma2013auto,
  title={Auto-encoding variational bayes},
  author={Kingma, Diederik P and Welling, Max},
  journal={arXiv preprint arXiv:1312.6114},
  year={2013}
}

@article{grill2020bootstrap,
  title={Bootstrap your own latent-a new approach to self-supervised learning},
  author={Grill, Jean-Bastien and Strub, Florian and Altch{\'e}, Florent and Tallec, Corentin and Richemond, Pierre and Buchatskaya, Elena and others},
  journal={NeurIPS},
  year={2020}
}

@article{dosovitskiy2020image,
  title={An image is worth 16x16 words: Transformers for image recognition at scale},
  author={Dosovitskiy, Alexey},
  journal={arXiv preprint arXiv:2010.11929},
  year={2020}
}

@inproceedings{higgins2017beta,
  title={beta-vae: Learning basic visual concepts with a constrained variational framework},
  author={Higgins, Irina and Matthey, Loic and Pal, Abhishek and Burgess, Christopher and Glorot, Xavier and Botvinick, Matthew and Mohamed, Shakir and Lerchner, Alexander},
  booktitle={ICLR},
  year={2017}
}

@inproceedings{locatello2019challenging,
  title={Challenging common assumptions in the unsupervised learning of disentangled representations},
  author={Locatello, Francesco and Bauer, Stefan and Lucic, Mario and Raetsch, Gunnar and Gelly, Sylvain and Sch{\"o}lkopf, Bernhard and Bachem, Olivier},
  booktitle={ICML},
  year={2019}
}

@inproceedings{jia2021scaling,
  title={Scaling up visual and vision-language representation learning with noisy text supervision},
  author={Jia, Chao and Yang, Yinfei and Xia, Ye and Chen, Yi-Ting and Parekh, Zarana and Pham, Hieu and Le, Quoc and Sung, Yun-Hsuan and Li, Zhen and Duerig, Tom},
  booktitle={ICML},
  year={2021}
}

@article{bronstein2021geometric,
  title={Geometric deep learning: Grids, groups, graphs, geodesics, and gauges},
  author={Bronstein, Michael M and Bruna, Joan and Cohen, Taco and Veli{\v{c}}kovi{\'c}, Petar},
  journal={arXiv preprint arXiv:2104.13478},
  year={2021}
}

@book{steenrod1999topology,
  title={The topology of fibre bundles},
  author={Steenrod, Norman Earl},
  volume={14},
  year={1999},
  publisher={Princeton university press}
}

@article{cover2006geometrical,
  title={Geometrical and statistical properties of systems of linear inequalities with applications in pattern recognition},
  author={Cover, Thomas M},
  journal={IEEE transactions on electronic computers},
  number={3},
  pages={326--334},
  year={2006},
  publisher={IEEE}
}

@article{cortes1995support,
  title={Support-vector networks},
  author={Cortes, Corinna and Vapnik, Vladimir},
  journal={Machine learning},
  volume={20},
  number={3},
  pages={273--297},
  year={1995},
  publisher={Springer}
}

@article{dai2024deepseekmoe,
  title={Deepseekmoe: Towards ultimate expert specialization in mixture-of-experts language models},
  author={Dai, Damai and Deng, Chengqi and Zhao, Chenggang and Xu, RX and Gao, Huazuo and Chen, Deli and Li, Jiashi and Zeng, Wangding and Yu, Xingkai and Wu, Yu and others},
  journal={arXiv preprint arXiv:2401.06066},
  year={2024}
}

@article{qiu2025gated,
  title={Gated Attention for Large Language Models: Non-linearity, Sparsity, and Attention-Sink-Free},
  author={Qiu, Zihan and Wang, Zekun and Zheng, Bo and Huang, Zeyu and Wen, Kaiyue and Yang, Songlin and Men, Rui and Yu, Le and Huang, Fei and Huang, Suozhi and others},
  journal={arXiv preprint arXiv:2505.06708},
  year={2025}
}

@article{brown2020language,
  title={Language models are few-shot learners},
  author={Brown, Tom and Mann, Benjamin and Ryder, Nick and Subbiah, Melanie and Kaplan, Jared D and Dhariwal, Prafulla and Neelakantan, Arvind and Shyam, Pranav and Sastry, Girish and Askell, Amanda and others},
  journal={Advances in neural information processing systems},
  volume={33},
  pages={1877--1901},
  year={2020}
}

@article{vaswani2017attention,
  title={Attention is all you need},
  author={Vaswani, Ashish and Shazeer, Noam and Parmar, Niki and Uszkoreit, Jakob and Jones, Llion and Gomez, Aidan N and Kaiser, {\L}ukasz and Polosukhin, Illia},
  journal={Advances in neural information processing systems},
  volume={30},
  year={2017}
}

@inproceedings{radford2021learning,
  title={Learning transferable visual models from natural language supervision},
  author={Radford, Alec and Kim, Jong Wook and Hallacy, Chris and Ramesh, Aditya and Goh, Gabriel and Agarwal, Sandhini and others},
  booktitle={ICML},
  year={2021}
}

@misc{simeoni2025dinov3,
  title={{DINOv3}},
  author={Sim{\'e}oni, Oriane and Vo, Huy V. and Seitzer, Maximilian and Baldassarre, Federico and Oquab, Maxime and Jose, Cijo and Khalidov, Vasil and Szafraniec, Marc and Yi, Seungeun and Ramamonjisoa, Micha{\"e}l and Massa, Francisco and Haziza, Daniel and Wehrstedt, Luca and Wang, Jianyuan and Darcet, Timoth{\'e}e and Moutakanni, Th{\'e}o and Sentana, Leonel and Roberts, Claire and Vedaldi, Andrea and Tolan, Jamie and Brandt, John and Couprie, Camille and Mairal, Julien and J{\'e}gou, Herv{\'e} and Labatut, Patrick and Bojanowski, Piotr},
  year={2025},
  eprint={2508.10104},
  archivePrefix={arXiv},
  primaryClass={cs.CV},
  url={https://arxiv.org/abs/2508.10104},
}

@article{tschannen2025siglip,
  title={Siglip 2: Multilingual vision-language encoders with improved semantic understanding, localization, and dense features},
  author={Tschannen, Michael and Gritsenko, Alexey and Wang, Xiao and Naeem, Muhammad Ferjad and Alabdulmohsin, Ibrahim and Parthasarathy, Nikhil and Evans, Talfan and Beyer, Lucas and Xia, Ye and Mustafa, Basil and others},
  journal={arXiv preprint arXiv:2502.14786},
  year={2025}
}

@inproceedings{he2022masked,
  title={Masked autoencoders are scalable vision learners},
  author={He, Kaiming and Chen, Xinlei and Xie, Saining and Li, Yanghao and Doll{\'a}r, Piotr and Girshick, Ross},
  booktitle={CVPR},
  year={2022}
}

@inproceedings{chen2020simple,
  title={A simple framework for contrastive learning of visual representations},
  author={Chen, Ting and Kornblith, Simon and Norouzi, Mohammad and Hinton, Geoffrey},
  booktitle={International conference on machine learning},
  pages={1597--1607},
  year={2020},
  organization={PmLR}
}

@inproceedings{he2020momentum,
  title={Momentum contrast for unsupervised visual representation learning},
  author={He, Kaiming and Fan, Haoqi and Wu, Yuxin and Xie, Saining and Girshick, Ross},
  booktitle={CVPR},
  year={2020}
}

@article{soatto2014visual,
  title={Visual representations: Defining properties and deep approximations},
  author={Soatto, Stefano and Chiuso, Alessandro},
  journal={arXiv preprint arXiv:1411.7676},
  year={2014}
}

@article{naitzat2020topology,
  title={Topology of deep neural networks},
  author={Naitzat, Gregory and Zhitnikov, Andrey and Lim, Lek-Heng},
  journal={Journal of Machine Learning Research},
  volume={21},
  number={184},
  pages={1--40},
  year={2020}
}

@article{oord2018cpc,
  title={Representation learning with contrastive predictive coding},
  author={Oord, Aaron van den and Li, Yazhe and Vinyals, Oriol},
  journal={arXiv preprint arXiv:1807.03748},
  year={2018}
}

@inproceedings{chen2020simclr,
  title={A simple framework for contrastive learning of visual representations},
  author={Chen, Ting and Kornblith, Simon and Norouzi, Mohammad and Hinton, Geoffrey},
  booktitle={International conference on machine learning},
  pages={1597--1607},
  year={2020},
  organization={PMLR}
}

@inproceedings{bardes2022vicreg,
  title={VICReg: Variance-Invariance-Covariance Regularization for Self-Supervised Learning},
  author={Bardes, Adrien and Ponce, Jean and LeCun, Yann},
  booktitle={International Conference on Learning Representations},
  year={2022}
}

@inproceedings{tian2020makes,
  title={What makes for good views for contrastive learning?},
  author={Tian, Yonglong and Sun, Chen and Poole, Ben and Krishnan, Dilip and Schmid, Cordelia and Isola, Phillip},
  booktitle={Advances in Neural Information Processing Systems},
  volume={33},
  pages={6827--6839},
  year={2020}
}

@inproceedings{wang2022ssl,
  title={Chaos is a Ladder: A New Theoretical Understanding of Contrastive Learning via Augmentation Overlap},
  author={Wang, Tongzhou and Isola, Phillip},
  booktitle={International Conference on Learning Representations},
  year={2022}
}
\bibliographystyle{icml2026}

\end{document}